\documentclass[10pt]{article}

\usepackage[preprint]{tmlr}
\usepackage{microtype}
\usepackage{amsmath,amssymb,amsthm,mathtools,bm}
\usepackage{booktabs,array,tabularx,multirow,longtable}
\usepackage{graphicx}
\usepackage{float}
\usepackage[section]{placeins}
\usepackage{flafter}
\usepackage{xcolor}
\usepackage{enumitem}
\usepackage[hidelinks]{hyperref}
\usepackage[nameinlink,noabbrev]{cleveref}
\usepackage{tikz}
\usepackage{pgfplots}
\usepgfplotslibrary{groupplots}
\pgfplotsset{compat=1.18}

\definecolor{bargegreen}{HTML}{A5CF83}
\definecolor{bargeyellow}{HTML}{F0E76F}
\definecolor{bargegold}{HTML}{ECB65F}
\definecolor{bargeorange}{HTML}{E89951}
\colorlet{bargedarkgreen}{bargegreen!52!black}
\colorlet{bargedarkyellow}{bargeyellow!62!black}
\colorlet{bargedarkgold}{bargegold!72!black}
\colorlet{bargedarkorange}{bargeorange!68!black}
\colorlet{bargegrid}{bargegreen!24}
\setlist[itemize]{leftmargin=1.6em,itemsep=0.25em,topsep=0.4em}

\renewcommand{\headrulewidth}{0pt}
\fancypagestyle{arxivfirstpage}{%
  \fancyhf{}%
  \fancyfoot[L]{\small Preprint}%
  \fancyfoot[C]{\thepage}%
  \renewcommand{\headrulewidth}{0pt}%
  \renewcommand{\footrulewidth}{0pt}%
}

\newtheorem{proposition}{Proposition}

\newcommand{\R}{\mathbb{R}}
\newcommand{\Cset}{\{1,\ldots,C\}}
\newcommand{\one}{\mathbb{I}}
\newcommand{\softmax}{\operatorname{softmax}}
\newcommand{\sg}{\operatorname{sg}}
\newcommand{\BARGE}{\textsc{Barge}}
\newcommand{\best}[1]{\mathbf{#1}}
\newcommand{\second}[1]{\underline{#1}}
\newcommand{\yes}{\checkmark}

\title{Bounded Adjustment with Reliability-Guided Embedding\\ for Imbalanced Learning with Noisy Labels}
\author{%
\name Mushir Akhtar$^{1}$ \quad Akarsh J.$^{2}$ \quad M. Tanveer$^{1}$ \quad Mohd. Arshad$^{1}$ \\
\addr $^{1}$Department of Mathematics, Indian Institute of Technology Indore, India \\
\addr $^{2}$Department of Electrical Engineering, Indian Institute of Technology Indore, India}

\begin{document}
\raggedbottom
\maketitle
\thispagestyle{arxivfirstpage}

\begin{abstract}
Class-balanced learning and label noise create a coupled failure mode: frequency correction is needed to prevent majority classes from dominating the decision rule, but it can magnify the influence of incorrectly labeled minority examples. We introduce \BARGE{} (Bounded Adjustment with Reliability-Guided Embeddings), a single-stage objective that combines a bounded, prior-adjusted density-power score with reliability-guided angular geometry. The classification score is strictly proper in the adjusted probability space and recovers the balanced Bayes ordering under clean supervision and the true class prior. Under supplied-label contamination, its finite range bounds the change in classification risk at any fixed predictor, and its logit gradient redescends when the model confidently contradicts the supplied label. The adjusted target probability also provides a detached reliability weight for class-equal feature compactness, while a one-sided separation term discourages positively aligned class directions. The resulting objective requires neither a noise rate nor a transition matrix, uses one network, and leaves inference unchanged. We evaluate \BARGE{} on CIFAR-10, CIFAR-100, and Tiny ImageNet under long-tail and step imbalance, clean supervision, and 20\% and 40\% random incorrect-label replacement. Across 12 clean-label settings, \BARGE{} is competitive with strong imbalance-learning baselines, ranking second overall and attaining the lowest error in four settings. Under label corruption, it achieves the lowest mean balanced error in all six dataset--corruption settings and reduces the six-setting average from 72.32\% for the strongest competing method to 70.00\%. It also attains the highest macro-F1 and macro-AUPRC in every corrupted-label setting. Component ablations show that class-equal angular compactness improves on the bounded score alone. These findings identify bounded predictive influence and reliability-guided geometry as complementary mechanisms for learning from imbalanced data with uncertain labels.
\end{abstract}

\section{Introduction}
\label{sec:introduction}

Balanced classification asks a learner to treat every class as equally consequential even when the training set does not. This objective is common in medical screening, fault recognition, and ecological monitoring, where the event of interest is often the least frequent one. It is also a distinct statistical target. Empirical risk under the observed training prior favors head classes, whereas balanced error corresponds to a uniform target prior. The mismatch changes the Bayes decision rule and the dynamics by which class-specific features are learned \citep{he2009learning,buda2018systematic,francazi2023theoretical}. In applications with asymmetric consequences, long-tailed learning may additionally require an explicit decision utility rather than a single aggregate accuracy measure \citep{li2025reliable}.

Modern long-tail learning attacks this mismatch at several levels. Reweighting modifies the empirical contribution of each class \citep{cui2019class,ren2018reweight}; margin and logit methods reshape class decisions \citep{cao2019learning,menon2021longtail}; decoupled and contrastive methods strengthen representations before or alongside classifier balancing \citep{kang2020decoupling,zhu2022balanced}. Recent work has moved beyond heuristic corrections. Generalized logit-adjusted and class-aware losses establish balanced-risk consistency for broad hypothesis classes \citep{cortes2025improved}, while an imbalance-sensitive margin framework provides strong $\mathcal H$-consistency guarantees and class-sensitive generalization bounds \citep{cortes2025balancing}. Difficulty-aware margins refine which minority examples receive emphasis \citep{son2025difficulty}, and model rebalancing directly reallocates parameter capacity toward tail classes \citep{luo2025model}. These advances clarify how to correct class frequency. They leave a second question unresolved: what should happen when the emphasized label is wrong?

That question matters because imbalance and annotation error reinforce one another. Rare classes are often assembled from weak retrieval, limited expert review, or heterogeneous data sources. An inverse-frequency weight can then amplify an incorrect minority label. Cross-entropy has a different failure mode: its gradient does not vanish when the network becomes confident that the supplied label is wrong. Noisy-label learning addresses these effects through transition correction, sample selection, co-training, or robust losses \citep{patrini2017loss,han2018coteaching,zhang2018gce,ma2020normalized}. Recent methods jointly address long-tail imbalance and label noise through label correction, expert specialization, or multimodal consensus \citep{chen2025robustla,li2025ibc,li2026care}. This motivates a complementary loss-level question: can predictive and representation-level influence be controlled within a single-network training protocol, without estimating a noise rate or transition matrix?

To address this setting, we introduce \BARGE{} (Bounded Adjustment with Reliability-Guided Embeddings), a single-stage objective that combines a prior-adjusted density-power proper score with reliability-guided angular geometry. Under clean supervision, propriety recovers the training posterior in the adjusted probability space, while the raw logits retain the balanced Bayes ordering after removal of the prior contribution. Under label corruption, boundedness and a redescending gradient limit the influence of a supplied label that strongly conflicts with the current evidence; these properties do not imply recovery of the clean posterior under arbitrary corruption. The adjusted target probability also serves as a detached reliability signal for representation learning. It modulates angular contraction toward the supplied class direction, while within-class normalization maintains class-equal geometric influence. A complementary one-sided penalty separates positively aligned classifier directions. Thus, the predictive and geometric components are coupled through a common measure of agreement between the model and the supplied label.


The paper makes the following key contributions.
\begin{itemize}
  \item We derive a prior-adjusted density-power objective that is strictly proper in the adjusted probability vector and, under clean supervision and the true prior, induces the balanced Bayes rule in the raw logits.
  \item We establish a finite per-sample range, a bounded classification-risk perturbation at a fixed predictor under label contamination, and a redescending logit gradient in the confident-contradiction limit. These are targeted influence properties, not a blanket guarantee under arbitrary label noise.
  \item We introduce reliability-guided, class-equal angular compactness together with one-sided classifier separation. Both geometric terms are bounded and avoid a quadratic sample-pair matrix.
  \item We evaluate the complete objective under clean and corrupted supervision using a common architecture, optimization procedure, data partitions, tuning protocol, and metric suite. The study separates clean balanced classification, scaling to 200 classes, label-corruption behavior, calibration, computational cost, and the contributions of the score and geometric terms.
\end{itemize}

The remainder of the paper is structured as follows. \Cref{sec:related} reviews related work; \cref{sec:problem} formalizes the problem; \cref{sec:method,sec:theory} present \BARGE{} and its theoretical properties; \cref{sec:experiments} reports the empirical evaluation; and \cref{sec:discussion,sec:limitations,sec:conclusion} provide the discussion, limitations, and conclusions.

\section{Related Work}
\label{sec:related}

\paragraph{Balanced classification.}
Sampling and loss weighting are the most direct corrections for skewed class frequencies. Class-balanced loss replaces raw frequency with an effective-number estimate \citep{cui2019class}, meta-reweighting learns example contributions from validation data \citep{ren2018reweight}, and Balanced Meta-Softmax embeds the class prior in the softmax formulation \citep{ren2020balanced}. Focal loss was introduced for foreground--background imbalance and downweights easy predictions rather than classes themselves \citep{lin2017focal}. Equalization losses and related gradient controls suppress head-class domination in detection \citep{tan2020equalization}. Margin-based methods such as LDAM enlarge minority margins \citep{cao2019learning}, while logit adjustment directly corrects the decision rule under prior shift \citep{menon2021longtail}. Cortes et al.\ \citep{cortes2025improved} generalize logit-adjusted and class-aware losses and prove balanced-risk consistency under explicit conditions. A complementary margin framework establishes strong $\mathcal H$-consistency and class-sensitive generalization guarantees \citep{cortes2025balancing}. A recent high-dimensional analysis further shows that the effect of reweighting on learned features depends on the symmetry and geometry of the problem \citep{obuchi2025reweighting}. Recent alternatives also intervene outside the loss: distribution-matched synthesis augments minority classes \citep{li2025synthesis}, while model rebalancing redistributes parameter capacity \citep{luo2025model}. \BARGE{} shares the concern with the population target, but adds bounded sample influence and feature geometry.

\paragraph{Long-tail representations.}
Class imbalance affects the learned representation as well as the final classifier. Decoupled training shows that a strong representation can coexist with a biased classifier and can often be rebalanced in a second stage \citep{kang2020decoupling}. Balanced contrastive learning modifies class complements and prototypes to improve minority structure \citep{zhu2022balanced}. Balanced information bottlenecks combine loss rebalancing and self-distillation across multiple feature levels to preserve label-relevant information under long-tailed sampling \citep{lan2025mbib}. Center loss and angular-margin objectives likewise demonstrate the value of compact within-class features and separated class directions \citep{wen2016center,liu2017sphereface,wang2018cosface,deng2019arcface}. Supervised contrastive learning provides a general pairwise route to discriminative geometry \citep{khosla2020supervised}. \BARGE{} uses the final classifier directions as class anchors. This avoids sample-pair storage and allows the geometry assigned to an observed label to be moderated by its current reliability.

\paragraph{Learning from corrupted labels.}
Noise-transition correction is statistically appealing when the transition matrix can be identified \citep{patrini2017loss}. Co-teaching instead relies on two networks exchanging small-loss examples \citep{han2018coteaching}; peer loss constructs comparisons without knowing the noise rates \citep{liu2020peer}. Robust-loss analyses connect symmetry, boundedness, and normalization to tolerance under specified noise models \citep{ghosh2017robust,zhang2018gce,ma2020normalized}. Later work emphasizes that robustness and reliable probability estimation are not the same property \citep{olmin2022robustness}. Representation-based analyses provide another perspective: self-supervised features can distribute label noise more uniformly and preserve compact class structure before noisy-label fitting \citep{xu2025uniform}. Recent long-tailed noisy-label pipelines combine label correction with prior estimation, semantic expert training, or multimodal consensus \citep{chen2025robustla,li2025ibc,li2026care}. \BARGE{} takes a complementary objective-level approach: it does not estimate clean labels or a transition model, and uses adjusted confidence only as a stopped-gradient relative weight.

\paragraph{Positioning.}
\BARGE{} builds on four principles: prior correction for balanced decision rules, density-power divergence for robust estimation \citep{basu1998density}, proper scoring rules for probabilistic prediction \citep{gneiting2007proper}, and angular regularization for representation learning. The same prior-adjusted target probability defines both the bounded classification score and the detached reliability weight used in the compactness term. Consequently, predictive influence and label-conditioned feature contraction are governed by a common measure of agreement between the model and the supplied label. A complementary one-sided penalty discourages alignment between class directions. This coupling distinguishes \BARGE{} from approaches that apply prior correction, robust scoring, or feature regularization independently.

\section{Problem Setting and Design Requirements}
\label{sec:problem}

Let \(\mathcal{D}=\{(\bm{x}_i,\widetilde y_i)\}_{i=1}^{N}\) be a training sample, where \(\bm{x}_i\in\mathcal{X}\subseteq\R^{d_x}\) and the observed label \(\widetilde y_i\in\Cset\). Its unobserved clean counterpart is \(y_i\). The observed class count and empirical prior are
\begin{equation}
  n_c=\sum_{i=1}^{N}\one[\widetilde y_i=c],
  \qquad
  \widehat\pi_c=\frac{n_c}{N},
  \qquad
  \widehat{\bm\pi}\in\Delta^{C-1}.
  \label{eq:empirical-prior}
\end{equation}
We assume that every observed class is represented in the training distribution, so \(\widehat\pi_c>0\) for all \(c\in\Cset\).

Evaluation gives equal weight to the \(C\) clean classes. For a classifier \(g:\mathcal{X}\rightarrow\Cset\), the balanced risk is
\begin{equation}
  \mathcal{R}_{\mathrm{bal}}(g)
  =\frac{1}{C}\sum_{c=1}^{C}
  \Pr\!\left(g(X)\neq c\mid Y=c\right).
  \label{eq:balanced-risk}
\end{equation}
Write \(q_c(\bm{x})=\Pr(Y=c\mid X=\bm{x})\) and \(\pi_c=\Pr(Y=c)\) under the clean training distribution. For later use, let \(\widetilde q_c(\bm{x})=\Pr(\widetilde Y=c\mid X=\bm{x})\) and \(\widetilde\pi_c=\Pr(\widetilde Y=c)\) denote their observed-label counterparts. If the clean class-conditionals \(p(\bm{x}\mid y=c)\) are shared between training and evaluation and the target prior is uniform, the balanced Bayes rule is
\begin{equation}
  g^*_{\mathrm{bal}}(\bm{x})
  \in\arg\max_{c\in\Cset}\frac{q_c(\bm{x})}{\pi_c}.
  \label{eq:balanced-bayes}
\end{equation}

The network contains a feature map and a linear classifier,
\begin{equation}
  \bm{h}_i=f_{\theta}(\bm{x}_i)\in\R^d,
  \qquad
  \bm{z}_i=W\bm{h}_i+\bm b\in\R^C,
  \label{eq:model}
\end{equation}
where \(W=[\bm w_1^\top;\ldots;\bm w_C^\top]\in\R^{C\times d}\) and \(\bm b\in\R^C\). We seek raw logits whose ordering approaches \cref{eq:balanced-bayes}. At the same time, one low-probability observed label should neither create an unbounded sample loss nor pull its feature toward a class direction as strongly as a consistent label. The noise rate and transition matrix are unknown.

These constraints lead to three design requirements. Under clean supervision, the predictive term must have the correct balanced population target. For any supplied label, its response to confident contradiction must be bounded and eventually diminish. Finally, the feature regularizer must give classes equal influence while allowing examples within a class to contribute according to current evidence.

\section{Bounded Adjustment with Reliability-Guided Embeddings}
\label{sec:method}

\subsection{Overview}

The construction follows a single principle: class frequency determines the desired decision rule, while the model's agreement with a supplied label determines how strongly that label may reshape the representation. For a minibatch \(\mathcal B\) of size \(B\), \BARGE{} combines a prior-adjusted classification score, class-equal compactness, and class-direction separation:
\begin{equation}
  \mathcal L_{\mathrm{BARGE}}
  = \mathcal L_{\mathrm{cls}}
  + \eta\bigl(\mathcal L_{\mathrm{comp}}+\mathcal L_{\mathrm{sep}}\bigr),
  \qquad \eta > 0.
  \label{eq:barge}
\end{equation}
The coefficient \(\eta\) controls the overall strength of the geometric terms, while the class-count-dependent exponent defined in \cref{sec:density-power-score} adapts the classification response analytically to the number of classes. Classification supplies the statistical target, compactness organizes examples around their class directions, and separation prevents those directions from becoming indistinguishable. \Cref{fig:mechanism} summarizes these roles.

\subsection{Prior-adjusted probability}

Class prevalence and class evidence play different roles. Adding the empirical log prior during training maps the raw logits to probabilities on the observed training distribution:
\begin{equation}
  r_{ic}
  =\frac{\widehat\pi_c\exp(z_{ic})}
  {\sum_{k=1}^{C}\widehat\pi_k\exp(z_{ik})},
  \qquad
  \bm r_i=\softmax(\bm z_i+\log\widehat{\bm\pi})
  \in\Delta^{C-1}.
  \label{eq:adjusted-probability}
\end{equation}
Prediction uses \(\arg\max_c z_{ic}\) without the adjustment. Under clean supervision, a proper score can therefore identify the training posterior in \(\bm r\), while the raw logits remove the prior contribution required for balanced decisions. Under corrupted supervision, \(\bm r\) targets the observed-label posterior instead; the distinction is made explicit in \cref{sec:theory}.

\subsection{Bounded density-power score}\label{sec:density-power-score}

For \(\bm r\in\Delta^{C-1}\), supplied label \(y\in\Cset\), and exponent \(\beta>0\), define
\begin{equation}
  \ell_{\beta}(\bm r,y)
  =\frac{1}{\beta}
  +\sum_{c=1}^{C}r_c^{1+\beta}
  -\frac{1+\beta}{\beta}r_y^{\beta}.
  \label{eq:classification-loss}
\end{equation}
The minibatch classification term is
\begin{equation}
  \mathcal L_{\mathrm{cls}}
  =\frac{1}{B}\sum_{i\in\mathcal B}
  \ell_{\beta_C}(\bm r_i,\widetilde y_i),
  \qquad
  \beta_C=\min\!\left\{\frac12,\frac{1}{\log C}\right\}.
  \label{eq:classification-batch}
\end{equation}
The density-power score is proper, like cross-entropy, but has finite range and a gradient that vanishes under confident contradiction. The exponent controls how quickly this attenuation begins. Its decrease with \(C\) avoids suppressing a small target probability too aggressively when probability mass is distributed over more classes.

\subsection{Reliability-guided class-equal compactness}

The classification score controls the logits but does not directly organize the feature space. For nonzero features and classifier rows, we use each normalized classifier row as a class direction and normalize the feature,
\begin{equation}
  \widehat{\bm h}_i=\frac{\bm h_i}{\lVert\bm h_i\rVert_2}\in\R^d,
  \qquad
  \widehat{\bm w}_c=\frac{\bm w_c}{\lVert\bm w_c\rVert_2}\in\R^d,
  \label{eq:normalization}
\end{equation}
for nonzero vectors. The adjusted probability assigned to the observed label supplies the reliability weight
\begin{equation}
  \omega_i=\sg\!\left(r_{i\widetilde y_i}^{\beta_C}\right)\in(0,1].
  \label{eq:reliability}
\end{equation}
The stop-gradient operator makes the weight a measure of current agreement rather than an additional target for the probability vector. A supplied label that agrees with the adjusted prediction may shape the embedding strongly; a confidently contradicted label receives little relative influence.

Let \(\mathcal B_c=\{i\in\mathcal B:\widetilde y_i=c\}\) and \(\mathcal C_{\mathcal B}=\{c:|\mathcal B_c|>0\}\). The compactness term is
\begin{equation}
  \mathcal L_{\mathrm{comp}}
  =\frac{1}{|\mathcal C_{\mathcal B}|}
  \sum_{c\in\mathcal C_{\mathcal B}}
  \frac{\sum_{i\in\mathcal B_c}\omega_i
  \bigl(1-\widehat{\bm h}_i^\top\widehat{\bm w}_c\bigr)}
  {\sum_{i\in\mathcal B_c}\omega_i}.
  \label{eq:compactness}
\end{equation}
Every class present in the minibatch contributes once, independent of its sample count. Within that class, a label that agrees with the adjusted prediction receives more influence than one that does not. The classwise denominator is essential: a global average would allow head classes to dominate again, whereas classwise self-normalization preserves equal class contribution.

\subsection{One-sided classifier separation}

Compactness alone can pull different classes toward nearby classifier directions. We therefore penalize acute off-diagonal angles in the normalized classifier Gram matrix:
\begin{equation}
  \mathcal L_{\mathrm{sep}}
  =\frac{1}{C(C-1)}
  \sum_{\substack{c,k=1\\c\neq k}}^{C}
  \left[\max\!\left(0,
  \widehat{\bm w}_c^\top\widehat{\bm w}_k\right)\right]^2.
  \label{eq:separation}
\end{equation}
The penalty acts only when two class directions have positive cosine similarity. Orthogonal and obtuse directions are left unchanged. This one-sided form creates separation where overlap is most plausible without imposing an infeasible orthogonality requirement when \(C>d\).

\begin{figure}[!htbp]
\centering
\begin{tikzpicture}
\pgfmathdeclarefunction{sigmoid}{1}{\pgfmathparse{1/(1+exp(-#1))}}
\begin{groupplot}[
  group style={group size=2 by 2,horizontal sep=1.1cm,vertical sep=2.0cm},
  width=0.455\textwidth,
  height=0.285\textwidth,
  grid=major,
  grid style={bargegrid},
  tick label style={font=\footnotesize},
  label style={font=\small},
  title style={font=\small},
  legend style={font=\scriptsize,draw=none,fill=none,cells={anchor=west}},
  every axis plot/.append style={line width=1.25pt}
]
\nextgroupplot[
  title={(a) Bounded predictive penalty},
  xlabel={Supplied-label probability $r_y$},
  ylabel={Per-sample loss},
  xmin=0.005,xmax=0.995,ymin=0,ymax=5.5,
  legend pos=north east
]
\addplot[bargedarkgreen,domain=0.005:0.995,samples=240] {-ln(x)};
\addlegendentry{Cross-entropy}
\addplot[bargedarkorange,domain=0.005:0.995,samples=240]
  {1/0.3 + x^1.3 + (1-x)^1.3 - (1.3/0.3)*x^0.3};
\addlegendentry{Density-power, $\beta=0.3$}

\nextgroupplot[
  title={(b) Redescending target response},
  xlabel={Binary margin $u_y-u_m$},
  xmin=-12,xmax=12,ymin=0,ymax=1.05,
  legend pos=north east
]
\addplot[bargedarkgreen,domain=-12:12,samples=260] {1-sigmoid(x)};
\addlegendentry{Cross-entropy}
\addplot[bargedarkorange,domain=-12:12,samples=260]
  {1.3*abs(sigmoid(x)^1.3*(1-sigmoid(x))
  -sigmoid(x)*(1-sigmoid(x))^1.3
  -sigmoid(x)^0.3*(1-sigmoid(x)))};
\addlegendentry{Density-power, $\beta=0.3$}

\nextgroupplot[
  title={(c) Reliability across class counts},
  xlabel={Adjusted target probability $r_y$},
  ylabel={Reliability weight $r_y^{\beta_C}$},
  xmin=0,xmax=1,ymin=0,ymax=1.02,
  legend pos=south east
]
\addplot[bargedarkgreen,domain=0.001:1,samples=240] {x^0.4343};
\addlegendentry{$C=10$}
\addplot[bargedarkyellow,domain=0.001:1,samples=240,dashed] {x^0.2171};
\addlegendentry{$C=100$}
\addplot[bargedarkorange,domain=0.001:1,samples=240,dashdotted] {x^0.1887};
\addlegendentry{$C=200$}

\nextgroupplot[
  title={(d) Angular geometry},
  xlabel={Angle (degrees)},
  xmin=0,xmax=180,ymin=0,ymax=2.05,
  xtick={0,45,90,135,180},
  legend pos=north west
]
\addplot[bargedarkgreen,domain=0:180,samples=240] {1-cos(x)};
\addlegendentry{Compactness $1-\cos\theta$}
\addplot[bargedarkorange,domain=0:180,samples=240,dashed] {max(0,cos(x))^2};
\addlegendentry{Separation $[\cos\phi]_+^2$}
\end{groupplot}
\end{tikzpicture}
\caption{The four mechanisms in \BARGE{}. The density-power score remains finite as the supplied-label probability approaches zero (a), and its target response vanishes under confident contradiction (b). The fixed exponent transfers adjusted confidence into a smooth reliability weight (c). Feature compactness acts over the full angular range, whereas class-direction separation acts only on acute angles (d). (a) and (b) use a binary probability slice for visualization; the propositions hold for any \(C\).}
\label{fig:mechanism}
\end{figure}
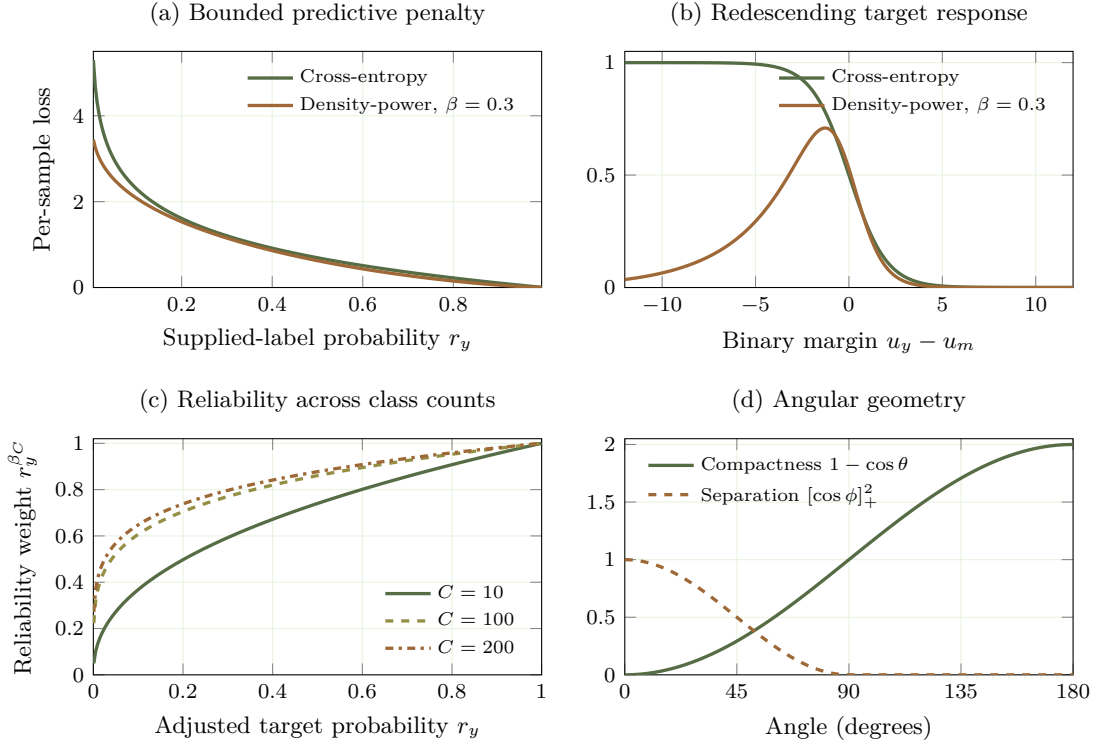

\begin{table}[H]
\centering
\caption{Properties of representative loss functions. ``Proper'' means strictly proper in the probability vector optimized by the classification score. Boundedness and redescending behavior for GCA apply to its generalized-cross-entropy form with \(q>0\); at \(q=0\), it reduces to an unbounded weighted cross-entropy form. Feature compactness and class-direction separation refer to explicit terms in the loss.}
\label{tab:properties}
\setlength{\tabcolsep}{3.8pt}
\small
\begin{tabular}{lcccccc}
\toprule
Method & Frequency & Proper & Bounded & Redescending & Feature & Direction \\
       & correction & score & sample score & contradiction & compactness & separation \\
\midrule
CE             & --   & \yes & --   & --   & --   & -- \\
Focal             & --   & --   & --   & --   & --   & -- \\
CB                & \yes & --   & --   & --   & --   & -- \\
LDAM           & \yes & --   & --   & --   & --   & -- \\
LA           & \yes & \yes & --   & --   & --   & -- \\
GCA          & \yes & --   & \yes & \yes & --   & -- \\
\BARGE{} (ours)                               & \yes & \yes & \yes & \yes & \yes & \yes \\
\bottomrule
\end{tabular}
\end{table}

\subsection{Relation to representative losses}
\Cref{tab:properties} compares the properties of \BARGE{} with those of representative classification objectives.

\subsection{Complexity}
After logits are available, classification costs \(O(BC)\), compactness costs \(O(Bd)\), and separation costs \(O(C^2d)\). Additional memory is \(O(BC+C^2)\). No \(B\times B\) sample-similarity matrix is formed. The loss adds no inference-time parameters or operations.


\section{Theoretical Analysis}
\label{sec:theory}

\begin{proposition}[Properness and clean balanced decision]
\label{prop:proper}
For any \(\beta>0\) and conditional label distribution \(\bm s\in\Delta^{C-1}\), the conditional risk \(\mathbb E_{L\sim\bm s}[\ell_\beta(\bm r,L)]\) is uniquely minimized at \(\bm r=\bm s\). Under clean supervision, the prior-shift assumptions in \cref{sec:problem}, and an exact clean prior \(\bm\pi\), a population minimizer satisfies
\begin{equation}
  z_c^*(\bm x)=\log q_c(\bm x)-\log\pi_c+\kappa(\bm x),
  \label{eq:optimal-logit}
\end{equation}
where \(\kappa(\bm x)\) is independent of \(c\). Thus \(\arg\max_c z_c^*(\bm x)\) equals the balanced Bayes rule in \cref{eq:balanced-bayes}.
\end{proposition}

\begin{proof}
Let \(\Phi(\bm r)=\beta^{-1}\sum_c r_c^{1+\beta}\), which is strictly convex for \(\beta>0\). Its Bregman divergence satisfies
\begin{align}
  D_\Phi(\bm e_y,\bm r)
  &=\Phi(\bm e_y)-\Phi(\bm r)
  -\nabla\Phi(\bm r)^\top(\bm e_y-\bm r)\\
  &=\frac1\beta+\sum_c r_c^{1+\beta}
  -\frac{1+\beta}{\beta}r_y^\beta
  =\ell_\beta(\bm r,y).
\end{align}
Taking expectation under \(\bm s\) and subtracting the risk at \(\bm s\) gives
\begin{equation}
  \mathbb E_{\bm s}[\ell_\beta(\bm r,L)]
  -\mathbb E_{\bm s}[\ell_\beta(\bm s,L)]
  =D_\Phi(\bm s,\bm r)\geq0,
\end{equation}
with equality only at \(\bm r=\bm s\). Under clean supervision, \(\bm s=\bm q\). Hence \(\softmax(\bm z^*+\log\bm\pi)=\bm q\), which yields \cref{eq:optimal-logit} up to a class-independent additive constant. Taking an argmax completes the result.
\end{proof}

\begin{proposition}[Finite range and redescending contradiction]
\label{prop:robust}
For every \(\bm r\in\Delta^{C-1}\) and \(y\in\Cset\),
\begin{equation}
  0\leq\ell_\beta(\bm r,y)\leq\frac{1+\beta}{\beta}.
  \label{eq:loss-bound}
\end{equation}
Let \(\bm u=\bm z+\log\bm\pi\) and \(S_\beta(\bm r)=\sum_j r_j^{1+\beta}\). Then
\begin{equation}
  \frac{\partial\ell_\beta}{\partial u_k}
  =(1+\beta)\left[
  r_k\bigl(r_k^\beta-S_\beta(\bm r)+r_y^\beta\bigr)
  -\one[k=y]r_y^\beta
  \right].
  \label{eq:logit-gradient}
\end{equation}
If \(r_y\to0\) and \(r_m\to1\) for one \(m\neq y\), every component of \cref{eq:logit-gradient} converges to zero. The cross-entropy target gradient does not.

Moreover, let \(\bm s_\varepsilon=(1-\varepsilon)\bm s+\varepsilon\bm a\), where \(\bm a\in\Delta^{C-1}\) is any contaminating conditional label distribution. For every fixed \(\bm r\),
\begin{equation}
  \left|
  \mathbb E_{L\sim\bm s_\varepsilon}\ell_\beta(\bm r,L)
  -\mathbb E_{L\sim\bm s}\ell_\beta(\bm r,L)
  \right|
  \leq \varepsilon\frac{1+\beta}{\beta}.
  \label{eq:contamination-bound}
\end{equation}
\end{proposition}

\begin{proof}
Nonnegativity follows from the Bregman representation. Since \(\sum_c r_c^{1+\beta}\leq1\) and \(r_y^\beta\geq0\), the upper bound is \(1/\beta+1\). Differentiating through \(\partial r_j/\partial u_k=r_j(\one[j=k]-r_k)\) gives \cref{eq:logit-gradient}. In the stated limit, \(S_\beta(\bm r)\to1\), \(r_y^\beta\to0\), and every non-dominant probability vanishes. Direct substitution gives zero for \(k=m\), \(k=y\), and all remaining coordinates. Finally, both conditional risks in \cref{eq:contamination-bound} lie in \([0,(1+\beta)/\beta]\), and their mixture difference is \(\varepsilon\) times their difference. This proves the bound.
\end{proof}

\begin{proposition}[Bounded geometry and objective]
\label{prop:geometry-bound}
For every nonempty minibatch,
\begin{equation}
  0\leq\mathcal L_{\mathrm{comp}}\leq2,
  \qquad
  0\leq\mathcal L_{\mathrm{sep}}\leq1,
\end{equation}
and therefore
\begin{equation}
  0\leq\mathcal L_{\mathrm{BARGE}}
  \leq\frac{1+\beta_C}{\beta_C}+3\eta.
  \label{eq:objective-bound}
\end{equation}
\end{proposition}

\begin{proof}
Normalized inner products lie in \([-1,1]\), so every compactness distance lies in \([0,2]\). Each class contribution in \cref{eq:compactness} is a positive weighted average of such distances. Averaging across present classes preserves the interval. Every squared positive cosine in \cref{eq:separation} lies in \([0,1]\). Combining these bounds with \cref{eq:loss-bound} proves \cref{eq:objective-bound}.
\end{proof}

\subsection{What changes under corrupted supervision}
\label{sec:corrupted-supervision}

Let \(T_{ab}(\bm x)=\Pr(\widetilde Y=b\mid Y=a,X=\bm x)\) denote a possibly instance-dependent corruption mechanism. The observed-label posterior and prior are
\begin{equation}
  \widetilde q_b(\bm x)
  =\sum_{a=1}^{C}T_{ab}(\bm x)q_a(\bm x),
  \qquad
  \widetilde\pi_b=\mathbb E_X[\widetilde q_b(X)].
  \label{eq:observed-posterior}
\end{equation}
Proposition~\ref{prop:proper} then identifies \(\widetilde{\bm q}\), not \(\bm q\), because the training risk is formed from \(\widetilde Y\). With an exact observed prior, the raw population logits order \(\widetilde q_c(\bm x)/\widetilde\pi_c\). This ratio is not generally the clean balanced rule \(q_c(\bm x)/\pi_c\). The balanced-consistency result is therefore a clean-supervision statement and is not used to claim consistency under label corruption.

The noise-relevant conclusions are influence statements. Equation~\eqref{eq:contamination-bound} limits how much an arbitrary fraction of replaced labels can change the classification risk at any fixed predictor. Proposition~\ref{prop:robust} also shows that the gradient of a supplied label vanishes when the adjusted prediction confidently contradicts it. For compactness, the normalized share of example \(i\) within its observed class is
\begin{equation}
  a_i=\frac{\omega_i}{\sum_{j\in\mathcal B_{\tilde y_i}}\omega_j}.
  \label{eq:relative-geometric-influence}
\end{equation}
If another example in that class has weight at least \(\delta>0\), then \(a_i\leq\omega_i/\delta\to0\) as \(r_{i\widetilde y_i}\to0\). Thus a contradicted label loses relative geometric influence when the minibatch contains supporting evidence. The condition is essential: a singleton class, or a class in which every supplied label is unreliable, receives no such guarantee. These results motivate the corruption experiments without asserting universal noise tolerance or recovery of the clean posterior.

\section{Experiments}
\label{sec:experiments}

The experiments are organized around three questions. First, does the bounded score and geometry preserve strong balanced classification when labels are clean? Second, does the method scale from 10 to 200 classes? Third, do the influence controls translate into better relative performance when the training labels are corrupted?

\subsection{Experimental protocol}

\paragraph{Datasets and imbalance.}
We use CIFAR-10, CIFAR-100 \citep{krizhevsky2009learning}, and Tiny ImageNet-200 \citep{le2015tinyimagenet}. A class-balanced validation set is held out before the CIFAR training set is made imbalanced: 50 images per class for CIFAR-10 and 5 per class for CIFAR-100. Tiny ImageNet uses its 100,000-image training split and labeled 10,000-image validation split. Five training images per class form a balanced 1,000-image model-selection set; the released validation images are reserved for testing.

For \(C\) classes and maximum count \(n_{\max}\), the requested training counts are
\begin{align}
  n_c^{\mathrm{LT}}
  &=\max\!\left\{1,
  \left\lfloor n_{\max}\rho^{-(c-1)/(C-1)}\right\rfloor\right\},
  \label{eq:lt-counts}\\
  n_c^{\mathrm{step}}
  &=\begin{cases}
  n_{\max}, & c\leq C/2,\\
  \max\!\left\{1,\left\lfloor n_{\max}/\rho\right\rfloor\right\}, & c>C/2,
  \end{cases}
  \label{eq:step-counts}
\end{align}
with classes ordered from head to tail. Here, \(\rho\) denotes the nominal imbalance ratio supplied to the class-count construction. The achieved ratio may differ because class counts are integer-valued and are lower-bounded by one example per class. The lower clamp therefore retains every class even when \(\rho>n_{\max}\). We evaluate both profiles at \(\rho\in\{100,1000\}\), and \cref{tab:achieved-imbalance} reports the resulting count extrema and achieved ratios.

The primary test view is constructed from each official test split using the corresponding profile, value of \(\rho\), and one-example lower clamp, following \citet{cortes2025improved}. The available per-class maxima are 1000 for CIFAR-10, 100 for CIFAR-100, and 50 for Tiny ImageNet. The achieved test ratios at \(\rho=100\) and \(1000\) are therefore \((100,1000)\), \((100,100)\), and \((50,50)\), respectively. Metrics on the complete balanced official test split are retained as a secondary view.

\begin{table}[tbp]
\centering
\caption{Clean training-set class counts and achieved imbalance after the balanced validation holdout. The minimum count and achieved ratio are the same for the long-tail and step profiles; their intermediate class counts differ.}
\label{tab:achieved-imbalance}
\setlength{\tabcolsep}{7pt}
\small
\begin{tabular}{lccccc}
\toprule
& & \multicolumn{2}{c}{$\rho=100$} & \multicolumn{2}{c}{$\rho=1000$} \\
\cmidrule(lr){3-4}\cmidrule(lr){5-6}
Dataset & $n_{\max}$ & $n_{\min}$ & Achieved ratio & $n_{\min}$ & Achieved ratio \\
\midrule
CIFAR-10 & 4950 & 49 & 101.02 & 4 & 1237.50 \\
CIFAR-100 & 495 & 4 & 123.75 & 1 & 495.00 \\
Tiny ImageNet & 495 & 4 & 123.75 & 1 & 495.00 \\
\bottomrule
\end{tabular}
\end{table}

\paragraph{Training-label corruption.}
Noise experiments use long-tail training data with \(\rho=100\). At corruption rate \(\varepsilon\in\{0.2,0.4\}\), each training label is independently replaced with probability \(\varepsilon\). The replacement is sampled uniformly from the \(C-1\) incorrect classes. Validation and test labels remain unchanged. This protocol models unstructured annotation error. It does not corrupt image features and does not model class-dependent confusions.

All prior-based losses use counts of the observed training labels. A pseudocount \(\alpha=1\) is added to every class count only when at least one observed class count is zero; this condition does not occur in any reported split. Across the common corrupted-label partitions, the achieved observed-label ratios after 20\% and 40\% replacement are 13.17 and 5.65 for CIFAR-10, 19.65 and 8.82 for CIFAR-100, and 26.25 and 9.85 for Tiny ImageNet. These ratios differ from the clean construction because random replacement redistributes labels toward low-count classes. Every applicable method uses the same observed counts and priors.

\paragraph{Architecture and optimization.}
Every method uses ResNet-32 \citep{he2016deep} with \(d=64\) features and classifier \(W\in\R^{C\times64}\). Training lasts 200 epochs with SGD, initial learning rate 0.2, momentum 0.9, Nesterov acceleration, weight decay \(10^{-3}\), and cosine annealing to zero. Training and evaluation batch sizes are 1024. CIFAR images receive random \(32\times32\) crops after four-pixel padding and random horizontal flips. Tiny ImageNet uses random \(64\times64\) crops after eight-pixel padding and random horizontal flips. Dataset-specific channel normalization is applied. Data ordering is deterministic, and every method is evaluated on identical training, validation, and test partitions. For each run, the epoch with the lowest validation balanced error determines the model evaluated on the test set.

\paragraph{Baselines and tuning.}
The comparison includes cross-entropy (CE), inverse-frequency weighted cross-entropy (WCE), focal loss \citep{lin2017focal}, class-balanced loss (CB) \citep{cui2019class}, LDAM \citep{cao2019learning}, logit adjustment (LA) \citep{menon2021longtail}, and generalized class-aware loss (GCA) \citep{cortes2025improved}. To preserve a controlled loss-level comparison, the benchmark is restricted to objectives that operate within the same single-network training protocol; methods that require label refurbishment, multiple experts, or external semantic signals are outside this scope. All objectives are evaluated from the same \(C\)-dimensional classifier outputs and observed labels, and methods that depend on class statistics use the same class-count vector. The geometric terms of \BARGE{} additionally use the \(d\)-dimensional feature vector preceding the final classifier and the classifier weight matrix in \(\mathbb R^{C\times d}\), as defined in \cref{eq:compactness,eq:separation}. GCA uses the true-class scaling specified by \citet{cortes2025improved}, with a single global factor that normalizes its class weights to have unit expectation under the training distribution.

Hyperparameters for every method are selected by validation balanced error. Baseline selection is performed for each dataset--profile--ratio setting. The validation grids are focal \(\gamma\in\{0,0.1,\ldots,1.0,1.5,2,\ldots,10\}\) on clean data and \(\{0,1,2,5\}\) with corruption; CB \(\beta\in\{0.1,\ldots,0.9,0.99,0.999,0.9999\}\) on clean data and \(\{0.9,0.99,0.999,0.9999\}\) with corruption; LDAM margin scale in \((\{1,5\}\times\{10^k:k=-4,\ldots,3\})\cup\{10^4\}\) on clean data and \(\{0.1,0.5,1,5\}\) with corruption; and GCA \(q\in\{0,0.1,\ldots,0.9\}\) on clean data and \(\{0,0.1,0.3,0.5,0.7,0.9\}\) with corruption. LA uses \(\tau=1\).

The only validation-selected coefficient in \BARGE{} is \(\eta\), which is evaluated using the same selection criterion as the baseline hyperparameters. It is selected from \(\{0.03,0.1,0.3,\allowbreak 0.5,1.0\}\) at dataset level using seeds \(\{1001,1002,\allowbreak 1003\}\). The selected values are 1.0 for CIFAR-10, 0.3 for CIFAR-100, and 0.1 for Tiny ImageNet. The Tiny ImageNet corruption study independently evaluates the same grid across both corruption rates and also selects 0.1. Equation~\eqref{eq:classification-batch} fixes \(\beta_C\) analytically at 0.4343, 0.2171, and 0.1887 for CIFAR-10, CIFAR-100, and Tiny ImageNet, respectively. Final comparisons use three seeds, \(\{42,1126,\allowbreak 2025\}\), for every method. All objective calculations use 32-bit floating-point arithmetic, and the feature and classifier normalization denominators in \BARGE{} are lower-bounded by \(10^{-8}\).

\paragraph{Metrics and comparisons.}
The primary metric is mean balanced error (MBE),
\begin{equation}
  \mathrm{MBE}=100\left(1-\frac1C\sum_{c=1}^{C}\mathrm{Recall}_c\right),
  \label{eq:mbe}
\end{equation}
where lower is better. The clean benchmark table reports total balanced error,
\begin{equation}
  \mathrm{TBE}
  = \sum_{c=1}^{C}\bigl(1-\mathrm{Recall}_c\bigr)
  = \frac{C}{100}\mathrm{MBE},
  \label{eq:tbe}
\end{equation}
which ranges from \(0\) to \(C\).
Secondary metrics are macro-F1, one-vs-rest macro-AUPRC, tail recall over the lowest-frequency third of classes, worst-class recall, negative log-likelihood (NLL), multiclass Brier score \citep{brier1950verification}, and 15-bin expected calibration error (ECE) \citep{guo2017calibration}. Dispersion is the standard deviation across the three seeds. Tiny ImageNet paired comparisons use two-sided Wilcoxon signed-rank tests \citep{wilcoxon1945individual} with Holm correction \citep{holm1979simple}. The clean comparison contains 12 matched setting--seed blocks; the noise comparison contains six.

\subsection{Clean imbalanced classification}

\Cref{tab:clean-tbe} gives a useful boundary on the claim. LA is the strongest clean-data baseline overall and wins seven of the 12 setting means. \BARGE{} ranks second on aggregate, wins four settings, appears in the top two in eight, and in the top three in ten. Its clearest clean gains occur in the CIFAR-10 long-tail settings: relative to LA, TBE falls from 2.87 to 2.69 at \(\rho=100\) and from 5.93 to 5.38 at \(\rho=1000\). It also achieves the lowest step-imbalance TBE at \(\rho=1000\) on CIFAR-100 and Tiny ImageNet. LA remains clearly better on several step-imbalanced settings. This pattern is informative. Prior correction is already a strong solution when labels are clean; bounded influence and geometric regularization are helpful in selected regimes, but they do not erase the advantage of a simpler prior-adjusted classifier everywhere.

Tiny ImageNet provides a second view because each setting contains 200 classes. Across its 12 paired setting--seed blocks, \BARGE{} has the best average TBE rank, 2.67, and the second-lowest pooled mean TBE, 146.44, behind LA at 145.90. It beats GCA and WCE in all 12 blocks, Focal in 10, LDAM in 9, CE in 8, CB in 7, and LA in 6. After Holm correction, the paired Wilcoxon comparisons are significant only against GCA and WCE (\(p=0.0034\) each). On the complete balanced test view, \BARGE{} also has the best average ranks for TBE (2.17) and macro-F1 (2.08). Other methods lead macro-AUPRC and tail recall, and every method has zero worst-class recall in at least one severe setting. Scaling to 200 classes therefore preserves the method's balanced-error consistency, but does not produce uniform dominance across all minority metrics.

\begin{table}[tbp]
\centering
\caption{Clean-test total balanced error (TBE; lower is better) under long-tail and step imbalance. Each entry reports the mean $\pm$ standard deviation over three seeds. Best values are shown in bold and second-best values are underlined.}
\label{tab:clean-tbe}
\scriptsize
\resizebox{\textwidth}{!}{%
\begin{tabular}{lcccccc}
\toprule
& \multicolumn{3}{c}{Long-tail imbalance} & \multicolumn{3}{c}{Step imbalance} \\
\cmidrule(lr){2-4}\cmidrule(lr){5-7}
Method & CIFAR-10 & CIFAR-100 & Tiny ImageNet & CIFAR-10 & CIFAR-100 & Tiny ImageNet \\
\midrule
\multicolumn{7}{c}{\textit{Imbalance parameter $\rho=100$}} \\
\midrule
CE    & $3.25\!\pm\!0.20$ & $64.87\!\pm\!0.45$ & $147.80\!\pm\!2.27$ & $3.69\!\pm\!0.11$ & $62.56\!\pm\!0.39$ & $141.39\!\pm\!0.46$ \\
WCE   & $3.16\!\pm\!0.05$ & $73.00\!\pm\!0.93$ & $156.45\!\pm\!2.69$ & $5.07\!\pm\!0.24$ & $77.87\!\pm\!4.09$ & $170.53\!\pm\!6.11$ \\
Focal & $3.35\!\pm\!0.21$ & $65.37\!\pm\!0.96$ & $147.94\!\pm\!0.43$ & $3.70\!\pm\!0.34$ & $62.99\!\pm\!0.52$ & $141.56\!\pm\!0.20$ \\
CB    & $3.18\!\pm\!0.34$ & $64.27\!\pm\!0.13$ & $148.27\!\pm\!1.69$ & $3.62\!\pm\!0.07$ & $\second{62.01\!\pm\!0.23}$ & $\best{140.47\!\pm\!0.61}$ \\
LDAM  & $2.89\!\pm\!0.07$ & $63.24\!\pm\!0.03$ & $150.56\!\pm\!1.24$ & $3.78\!\pm\!0.28$ & $62.76\!\pm\!0.32$ & $\second{140.53\!\pm\!0.72}$ \\
LA    & $\second{2.87\!\pm\!0.14}$ & $\best{61.93\!\pm\!0.20}$ & $\best{142.85\!\pm\!1.61}$ & $\best{2.77\!\pm\!0.17}$ & $\best{57.06\!\pm\!0.45}$ & $140.75\!\pm\!0.56$ \\
GCA   & $4.59\!\pm\!0.14$ & $94.98\!\pm\!4.06$ & $190.71\!\pm\!3.61$ & $5.37\!\pm\!0.07$ & $77.35\!\pm\!18.75$ & $198.99\!\pm\!0.02$ \\
\BARGE{} & $\best{2.69\!\pm\!0.03}$ & $\second{62.20\!\pm\!0.23}$ & $\second{145.86\!\pm\!1.90}$ & $\second{3.41\!\pm\!0.14}$ & $62.32\!\pm\!0.34$ & $141.02\!\pm\!1.17$ \\
\midrule
\multicolumn{7}{c}{\textit{Imbalance parameter $\rho=1000$}} \\
\midrule
CE    & $6.36\!\pm\!0.10$ & $74.19\!\pm\!0.56$ & $162.79\!\pm\!1.75$ & $5.29\!\pm\!0.04$ & $62.79\!\pm\!0.20$ & $\second{140.19\!\pm\!0.31}$ \\
WCE   & $7.32\!\pm\!0.38$ & $90.93\!\pm\!2.60$ & $190.86\!\pm\!1.80$ & $7.90\!\pm\!0.24$ & $94.61\!\pm\!1.21$ & $195.50\!\pm\!0.96$ \\
Focal & $6.46\!\pm\!0.07$ & $75.82\!\pm\!0.82$ & $160.77\!\pm\!0.73$ & $5.34\!\pm\!0.04$ & $62.94\!\pm\!0.50$ & $140.84\!\pm\!0.54$ \\
CB    & $6.39\!\pm\!0.09$ & $74.12\!\pm\!0.21$ & $159.57\!\pm\!0.52$ & $\second{5.12\!\pm\!0.20}$ & $\second{62.52\!\pm\!0.27}$ & $140.88\!\pm\!0.09$ \\
LDAM  & $6.35\!\pm\!0.11$ & $74.15\!\pm\!0.18$ & $\second{158.43\!\pm\!1.20}$ & $5.39\!\pm\!0.06$ & $62.86\!\pm\!0.29$ & $140.79\!\pm\!0.44$ \\
LA    & $\second{5.93\!\pm\!0.40}$ & $\best{72.56\!\pm\!0.45}$ & $\best{157.92\!\pm\!0.70}$ & $\best{4.64\!\pm\!0.15}$ & $62.59\!\pm\!0.50$ & $142.06\!\pm\!1.61$ \\
GCA   & $7.58\!\pm\!0.14$ & $96.82\!\pm\!3.68$ & $197.85\!\pm\!2.00$ & $5.33\!\pm\!0.03$ & $65.57\!\pm\!0.80$ & $190.85\!\pm\!13.59$ \\
\BARGE{}  & $\best{5.38\!\pm\!0.32}$ & $\second{72.74\!\pm\!0.11}$ & $159.24\!\pm\!5.46$ & $5.23\!\pm\!0.05$ & $\best{62.26\!\pm\!0.29}$ & $\best{139.63\!\pm\!0.75}$ \\
\bottomrule
\end{tabular}%
}
\end{table}

\subsection{Long-tailed classification with label noise}

As summarized in \cref{tab:noise-mbe}, \BARGE{} achieves the lowest mean balanced error in every one of the six dataset--corruption settings, improving the six-setting average from 72.32\% (LA) to 70.00\%.

The largest margin is on CIFAR-10 with 20\% corruption, where MBE decreases by 5.21 points relative to LA and by 9.35 points relative to CB. The margin narrows to 0.25 points at 40\% corruption on CIFAR-10, consistent with the two methods converging as noise dominates the signal available to any prior-based correction. The CIFAR-100 margins are 0.45 points over LDAM at 20\% corruption and 2.01 points over CE at 40\%. On Tiny ImageNet, \BARGE{} improves over the strongest competing mean by 0.51 points at 20\% corruption and 1.74 points at 40\%. \BARGE{} ranks first in all six settings; a paired Wilcoxon test confirms this ordering is not due to chance in the uncorrected comparisons ($p = 0.0313$ per baseline), though with only six matched blocks, Holm-corrected significance is a high bar to clear.

\begin{table}[tbp]
\centering
\caption{Test mean balanced error (MBE, \%; lower is better) under long-tail imbalance with $\rho=100$ and random incorrect-label replacement. Each entry reports the mean $\pm$ standard deviation over three seeds. The final column averages the six dataset--corruption setting means.}
\label{tab:noise-mbe}
\setlength{\tabcolsep}{2.8pt}
\scriptsize
\resizebox{\textwidth}{!}{%
\begin{tabular}{lccccccc}
\toprule
& \multicolumn{2}{c}{CIFAR-10} & \multicolumn{2}{c}{CIFAR-100} & \multicolumn{2}{c}{Tiny ImageNet} & \\
\cmidrule(lr){2-3}\cmidrule(lr){4-5}\cmidrule(lr){6-7}
Method & 20\% & 40\% & 20\% & 40\% & 20\% & 40\% & Average \\
\midrule
CE    & $50.63\!\pm\!2.15$ & $61.16\!\pm\!1.18$ & $76.37\!\pm\!0.71$ & $\second{84.27\!\pm\!0.81}$ & $82.67\!\pm\!1.26$ & $\second{86.57\!\pm\!1.36}$ & 73.61 \\
WCE   & $56.52\!\pm\!1.31$ & $59.94\!\pm\!2.62$ & $82.54\!\pm\!0.98$ & $87.38\!\pm\!1.39$ & $86.66\!\pm\!1.14$ & $89.66\!\pm\!1.29$ & 77.12 \\
Focal & $51.97\!\pm\!0.64$ & $62.61\!\pm\!3.29$ & $77.33\!\pm\!0.79$ & $84.98\!\pm\!0.60$ & $82.38\!\pm\!2.04$ & $87.49\!\pm\!0.50$ & 74.46 \\
CB    & $49.80\!\pm\!0.49$ & $62.42\!\pm\!0.90$ & $76.20\!\pm\!0.18$ & $85.30\!\pm\!1.15$ & $82.57\!\pm\!0.38$ & $86.74\!\pm\!0.38$ & 73.84 \\
LDAM  & $50.34\!\pm\!0.67$ & $61.77\!\pm\!0.47$ & $\second{75.94\!\pm\!0.67}$ & $84.41\!\pm\!1.44$ & $\second{81.79\!\pm\!0.72}$ & $87.15\!\pm\!0.36$ & 73.57 \\
LA    & $\second{45.66\!\pm\!2.07}$ & $\second{55.97\!\pm\!0.77}$ & $76.67\!\pm\!0.31$ & $84.50\!\pm\!1.23$ & $83.48\!\pm\!1.11$ & $87.65\!\pm\!0.83$ & $\second{72.32}$ \\
GCA   & $50.70\!\pm\!5.01$ & $56.92\!\pm\!1.80$ & $90.06\!\pm\!0.62$ & $93.70\!\pm\!3.24$ & $99.22\!\pm\!0.39$ & $99.33\!\pm\!0.29$ & 81.66 \\
\BARGE{} & $\best{40.45\!\pm\!0.52}$ & $\best{55.72\!\pm\!2.15}$ & $\best{75.49\!\pm\!0.96}$ & $\best{82.26\!\pm\!0.88}$ & $\best{81.28\!\pm\!1.11}$ & $\best{84.83\!\pm\!0.47}$ & $\best{70.00}$ \\
\bottomrule
\end{tabular}%
}
\end{table}

To show how performance changes as corruption increases, \cref{fig:noise-trajectory} connects the clean and corrupted-label results for CE, LA, GCA, and \BARGE{}.

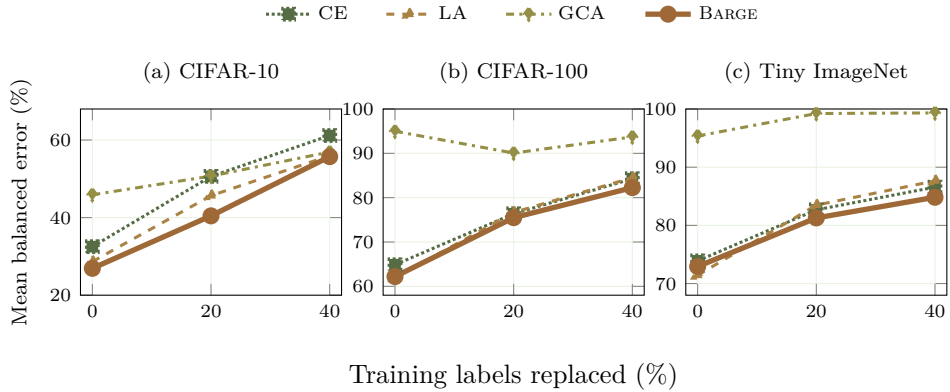
\begin{figure}[tbp]
\centering
\begin{tikzpicture}
\begin{groupplot}[
  group style={group size=3 by 1,horizontal sep=0.55cm},
  width=0.305\textwidth,
  height=0.245\textwidth,
  grid=major,
  grid style={bargegrid},
  tick label style={font=\scriptsize},
  label style={font=\footnotesize},
  title style={font=\footnotesize},
  legend style={font=\scriptsize,draw=none,fill=none},
  every axis plot/.append style={mark size=2.2pt,line width=1.15pt},
  xmin=-2,xmax=42,xtick={0,20,40}
]
\nextgroupplot[
  title={(a) CIFAR-10},
  ylabel={Mean balanced error (\%)},
  ymin=20,ymax=68,
  legend to name=noiselegend,
  legend columns=4,
  legend style={/tikz/every even column/.append style={column sep=1.2em}}
]
\addplot[bargedarkgreen,mark=square*,densely dotted] coordinates {(0,32.51) (20,50.63) (40,61.16)};
\addlegendentry{CE}
\addplot[bargedarkgold,mark=triangle*,dashed] coordinates {(0,28.65) (20,45.66) (40,55.97)};
\addlegendentry{LA}
\addplot[bargedarkyellow,mark=diamond*,dashdotted] coordinates {(0,45.89) (20,50.70) (40,56.92)};
\addlegendentry{GCA}
\addplot[bargedarkorange,mark=*,line width=2pt] coordinates {(0,26.92) (20,40.45) (40,55.72)};
\addlegendentry{\BARGE{}}

\nextgroupplot[
  title={(b) CIFAR-100},
  ymin=58,ymax=100
]
\addplot[bargedarkgreen,mark=square*,densely dotted] coordinates {(0,64.87) (20,76.37) (40,84.27)};
\addplot[bargedarkgold,mark=triangle*,dashed] coordinates {(0,61.93) (20,76.67) (40,84.50)};
\addplot[bargedarkyellow,mark=diamond*,dashdotted] coordinates {(0,94.98) (20,90.06) (40,93.70)};
\addplot[bargedarkorange,mark=*,line width=2pt] coordinates {(0,62.20) (20,75.49) (40,82.26)};

\nextgroupplot[
  title={(c) Tiny ImageNet},
  ymin=68,ymax=100
]
\addplot[bargedarkgreen,mark=square*,densely dotted] coordinates {(0,73.90) (20,82.67) (40,86.57)};
\addplot[bargedarkgold,mark=triangle*,dashed] coordinates {(0,71.43) (20,83.48) (40,87.65)};
\addplot[bargedarkyellow,mark=diamond*,dashdotted] coordinates {(0,95.36) (20,99.22) (40,99.33)};
\addplot[bargedarkorange,mark=*,line width=2pt] coordinates {(0,72.93) (20,81.28) (40,84.83)};
\end{groupplot}
\path (group c1r1.south west) --
  node[below=22pt] {Training labels replaced (\%)}
  (group c3r1.south east);
\path (group c1r1.north west) --
  node[above=25pt] {\ref{noiselegend}}
  (group c3r1.north east);
\end{tikzpicture}
\caption{Test mean balanced error under long-tail imbalance with $\rho=100$ as the proportion of randomly replaced training labels increases. Each point reports the mean over three seeds, and lines connect the evaluated corruption levels.}

\label{fig:noise-trajectory}
\end{figure}

\Cref{fig:noise-trajectory} makes the clean-to-noisy transition visible. On CIFAR-10, the gap between \BARGE{} and CE grows from 5.59 points on clean data to 10.18 points at 20\% corruption, then narrows at 40\%. On CIFAR-100, \BARGE{} begins 0.27 points behind LA, overtakes it at 20\%, and leads it by 2.24 points at 40\%. Tiny ImageNet repeats the crossover: \BARGE{} starts 1.50 points behind LA, then leads by 2.20 points at 20\% corruption and 2.82 points at 40\%. This clean-to-noisy crossover provides empirical evidence consistent with the bounded-contamination and redescending-influence properties; the accompanying theory characterizes influence under contamination rather than recovery of the clean balanced rule.

\subsection{Minority recognition and probability quality}

MBE alone does not show whether an improvement is distributed across minority classes or concentrated in a few easy ones. \Cref{tab:noise-discrimination} therefore reports macro and tail metrics for every corruption setting.

The CIFAR-10 gains are not confined to the average class. At 20\% corruption, \BARGE{} raises macro-F1 by 5.81 points over LA, macro-AUPRC by 5.87 points, tail recall by 9.86 points over WCE, and worst-class recall by 15.87 points. At 40\%, it remains first on all four metrics, although the macro-F1 and macro-AUPRC margins over GCA are 1.56 and 0.62 points. The CIFAR-100 view is harsher. \BARGE{} still leads macro-F1, macro-AUPRC, and tail recall at both corruption levels, but LA alone has nonzero worst-class recall at 20\%, and no method recovers every class at 40\%.

Tiny ImageNet separates average discrimination from extreme-tail recovery. \BARGE{} leads macro-F1 and macro-AUPRC at both corruption levels, including margins of 1.63 and 1.56 points at 20\%. It does not lead tail recall: LA is best at 20\%, WCE is best at 40\%, and \BARGE{} falls from 1.66\% to 0.06\%. Worst-class recall is zero for every method. The lower balanced error therefore comes mainly from broader gains outside the rarest classes, not from recovering every tail class. This qualification is central to interpreting the 200-class result. The regression is consistent with \cref{sec:corrupted-supervision}: within-class normalization offers no recovery guarantee when every observed label in a rare class is unreliable.

\begin{table}[]
\centering
\caption{Discrimination and minority-recognition metrics under label noise (higher is better). Values are percentages averaged over three seeds. Tail recall averages the lowest-frequency third of classes. Best and unique second-best entries in each column are bold and underlined.}
\label{tab:noise-discrimination}
\setlength{\tabcolsep}{2.6pt}
\scriptsize
\resizebox{\textwidth}{!}{%
\begin{tabular}{lcccccccc}
\toprule
& \multicolumn{4}{c}{20\% labels replaced} & \multicolumn{4}{c}{40\% labels replaced} \\
\cmidrule(lr){2-5}\cmidrule(lr){6-9}
Method & Macro-F1 & Macro-AUPRC & Tail recall & Worst recall & Macro-F1 & Macro-AUPRC & Tail recall & Worst recall \\
\midrule
\multicolumn{9}{c}{\textit{(a) CIFAR-10, long-tail $\rho=100$}} \\
\midrule
CE    & 46.19 & 50.67 & 25.26 & 4.53 & 32.18 & 36.58 & 4.07 & 0.00 \\
WCE   & 42.72 & 45.08 & $\second{40.07}$ & $\second{19.70}$ & 39.31 & 40.36 & $\second{26.88}$ & $\second{19.43}$ \\
Focal & 43.96 & 48.11 & 19.44 & 6.40 & 32.28 & 34.13 & 6.32 & 3.63 \\
CB    & 44.37 & 54.08 & 18.06 & 0.00 & 31.72 & 35.33 & 4.82 & 0.30 \\
LDAM  & 46.13 & 51.67 & 22.80 & 7.73 & 31.00 & 34.46 & 2.27 & 0.03 \\
LA    & $\second{52.64}$ & $\second{56.48}$ & 36.47 & 12.63 & 40.52 & 42.55 & 21.16 & 5.60 \\
GCA   & 46.83 & 50.51 & 29.21 & 12.87 & $\second{41.67}$ & $\second{43.69}$ & 24.63 & 12.87 \\
\BARGE{} & $\best{58.45}$ & $\best{62.35}$ & $\best{49.93}$ & $\best{35.57}$ & $\best{43.23}$ & $\best{44.31}$ & $\best{34.57}$ & $\best{20.70}$ \\
\midrule
\multicolumn{9}{c}{\textit{(b) CIFAR-100, long-tail $\rho=100$}} \\
\midrule
CE    & 20.45 & 19.19 & 5.59 & 0.00 & 12.49 & 11.76 & 1.74 & 0.00 \\
WCE   & 14.68 & 12.72 & 3.85 & 0.00 & 10.62 & 8.38 & 2.42 & 0.00 \\
Focal & 19.25 & 18.60 & 4.99 & 0.00 & 11.65 & 11.44 & 1.31 & 0.00 \\
CB    & 20.39 & 19.43 & 5.29 & 0.00 & 11.95 & 10.92 & 2.02 & 0.00 \\
LDAM  & 21.11 & 19.53 & 5.93 & 0.00 & 12.08 & 11.62 & 1.73 & 0.00 \\
LA    & $\second{22.72}$ & $\second{19.79}$ & $\second{8.99}$ & $\best{0.33}$ & $\second{12.83}$ & $\second{12.14}$ & $\second{2.44}$ & 0.00 \\
GCA   & 3.54 & 4.20 & 0.00 & 0.00 & 1.85 & 2.79 & 0.00 & 0.00 \\
\BARGE{} & $\best{23.11}$ & $\best{22.10}$ & $\best{9.66}$ & 0.00 & $\best{15.60}$ & $\best{14.18}$ & $\best{3.64}$ & 0.00 \\
\midrule
\multicolumn{9}{c}{\textit{(c) Tiny ImageNet, long-tail $\rho=100$}} \\
\midrule
CE    & 14.24 & 13.57 & 1.88 & 0.00 & 9.76 & 9.55 & 0.45 & 0.00 \\
WCE   & 11.95 & 9.97 & 1.99 & 0.00 & 8.51 & 6.70 & $\best{1.33}$ & 0.00 \\
Focal & 14.70 & 13.44 & $\second{2.33}$ & 0.00 & 9.26 & $\second{9.59}$ & 0.40 & 0.00 \\
CB    & 13.86 & 13.58 & 1.31 & 0.00 & $\second{9.95}$ & 9.00 & 0.61 & 0.00 \\
LDAM  & 14.88 & $\second{14.46}$ & 1.23 & 0.00 & 9.56 & 9.31 & 0.16 & 0.00 \\
LA    & $\second{14.95}$ & 14.15 & $\best{3.57}$ & 0.00 & 9.81 & 9.42 & 0.80 & 0.00 \\
GCA   & 0.05 & 0.72 & 0.72 & 0.00 & 0.01 & 0.65 & $\second{0.93}$ & 0.00 \\
\BARGE{} & $\best{16.58}$ & $\best{16.02}$ & 1.66 & 0.00 & $\best{11.13}$ & $\best{11.11}$ & 0.06 & 0.00 \\
\bottomrule
\end{tabular}%
}
\end{table}

The probability metrics in \cref{tab:noise-probability} expose a useful distinction.

At 20\% corruption, \BARGE{} is best on NLL, Brier score, and ECE for CIFAR-10, and on NLL and Brier score for CIFAR-100. At 40\% corruption, LA and WCE are marginally better calibrated on CIFAR-10, although the Brier difference between LA and \BARGE{} is only 0.0004. On CIFAR-100, GCA has the lowest ECE, while \BARGE{} has the lowest NLL and Brier score at both corruption levels. These results motivate reporting calibration together with distributional and class-balanced recognition metrics.

Tiny ImageNet is more uniform: \BARGE{} has the lowest NLL, Brier score, and ECE at both corruption levels. At 40\%, for example, its NLL is 4.457 versus 4.763 for LA, while ECE falls from 0.0877 to 0.0613. These gains agree with the MBE and macro metrics, even though tail recall remains weak. ECE is therefore interpreted beside NLL, Brier score, and class-balanced recognition rather than as a standalone measure of quality \citep{guo2017calibration,olmin2022robustness}.

\begin{table}[]
\centering
\caption{Probabilistic quality under label noise (lower is better). Values are means over three seeds. NLL is negative log-likelihood; ECE uses 15 equal-width confidence bins. Best and unique second-best entries in each column are bold and underlined.}
\label{tab:noise-probability}
\setlength{\tabcolsep}{5.0pt}
\scriptsize
\begin{tabular}{lcccccc}
\toprule
& \multicolumn{3}{c}{20\% labels replaced} & \multicolumn{3}{c}{40\% labels replaced} \\
\cmidrule(lr){2-4}\cmidrule(lr){5-7}
Method & NLL & Brier & ECE & NLL & Brier & ECE \\
\midrule
\multicolumn{7}{c}{\textit{(a) CIFAR-10, long-tail $\rho=100$}} \\
\midrule
CE    & 1.935 & 0.7456 & 0.2331 & 1.977 & 0.8115 & 0.1689 \\
WCE   & 1.855 & 0.7484 & 0.1335 & $\second{1.765}$ & 0.7587 & $\best{0.0917}$ \\
Focal & 2.045 & 0.7655 & 0.2380 & 2.100 & 0.8263 & 0.1518 \\
CB    & $\second{1.633}$ & 0.6860 & $\second{0.1305}$ & 2.053 & 0.8221 & 0.1634 \\
LDAM  & 2.027 & 0.7502 & 0.2515 & 2.143 & 0.8544 & 0.2315 \\
LA    & 1.639 & $\second{0.6550}$ & 0.1420 & $\best{1.735}$ & $\best{0.7388}$ & $\second{0.0925}$ \\
GCA   & 1.686 & 0.7062 & 0.1488 & 1.863 & 0.7717 & 0.1796 \\
\BARGE{} & $\best{1.414}$ & $\best{0.5850}$ & $\best{0.0897}$ & 1.767 & $\second{0.7392}$ & 0.0927 \\
\midrule
\multicolumn{7}{c}{\textit{(b) CIFAR-100, long-tail $\rho=100$}} \\
\midrule
CE    & 5.128 & 1.0919 & 0.3998 & 5.164 & 1.0780 & 0.2949 \\
WCE   & 6.530 & 1.0973 & 0.3467 & 7.418 & 1.1447 & 0.3625 \\
Focal & $\second{4.263}$ & $\second{0.9782}$ & 0.2388 & 4.491 & 0.9980 & 0.1569 \\
CB    & 5.285 & 1.1020 & 0.4122 & 5.235 & 1.0441 & 0.2157 \\
LDAM  & 5.532 & 1.1172 & 0.4268 & 5.175 & 1.0892 & 0.3091 \\
LA    & 4.944 & 1.0488 & 0.3519 & $\second{4.456}$ & 0.9947 & $\second{0.1224}$ \\
GCA   & 4.611 & 0.9898 & $\best{0.0839}$ & 4.627 & $\second{0.9902}$ & $\best{0.0467}$ \\
\BARGE{} & $\best{3.743}$ & $\best{0.9453}$ & $\second{0.1821}$ & $\best{4.111}$ & $\best{0.9730}$ & 0.1354 \\
\midrule
\multicolumn{7}{c}{\textit{(c) Tiny ImageNet, long-tail $\rho=100$}} \\
\midrule
CE    & 5.218 & 1.0679 & 0.3136 & 5.129 & 1.0329 & 0.2156 \\
WCE   & 6.732 & 1.0741 & 0.2782 & 6.999 & 1.0827 & 0.2466 \\
Focal & 5.028 & 1.0387 & 0.2729 & 4.919 & 0.9942 & 0.1235 \\
CB    & 5.017 & 1.0445 & 0.2747 & 5.453 & 1.0615 & 0.2519 \\
LDAM  & 5.307 & 1.0845 & 0.3425 & 5.047 & 1.0243 & 0.1848 \\
LA    & $\second{4.706}$ & $\second{0.9908}$ & $\second{0.1794}$ & $\second{4.763}$ & $\second{0.9836}$ & $\second{0.0877}$ \\
GCA   & 8.037 & 1.3284 & 0.4242 & 10.100 & 1.6548 & 0.6625 \\
\BARGE{} & $\best{4.180}$ & $\best{0.9478}$ & $\best{0.1120}$ & $\best{4.457}$ & $\best{0.9610}$ & $\best{0.0613}$ \\
\bottomrule
\end{tabular}
\end{table}

\subsection{Ablation Studies}
\label{sec:ablation}

The benchmark establishes between-method performance, but it does not identify which part of \BARGE{} is responsible for the observed gains. We therefore hold the CIFAR-100 long-tail setting fixed at \(\rho=100\), retain the common data partitions and final seeds \(\{42,1126,2025\}\), and repeat the 20\% and 40\% random incorrect-label replacement conditions. The coefficient \(\eta=0.3\) is the dataset-level value selected on disjoint tuning seeds and is held fixed for every geometry-enabled variant. No component-specific coefficient is tuned. \Cref{tab:ablation} reports the resulting component comparison.

The bounded score alone improves over CE at both corruption levels, and \cref{tab:ablation} identifies a repeated contribution from compactness. Relative to the bounded score, uniform compactness lowers MBE from 76.08\% to 74.72\% at 20\% corruption and from 82.40\% to 81.98\% at 40\%. It also raises macro-AUPRC from 21.15\% to 21.97\% and from 14.20\% to 14.90\%, respectively. These directions hold in all six matched seed--corruption comparisons, demonstrating that compactness repeatedly improves both class-balanced error and class-balanced ranking metrics when added to the bounded score.

Reliability weighting concentrates its strongest gains in calibration and ranking. Against uniform compactness, it lowers mean ECE from 15.63\% to 13.73\% at 20\% corruption and from 13.94\% to 12.77\% at 40\%, while attaining the best macro-AUPRC in the 20\% block. Separation produces the lowest MBE in both blocks. The full objective gives the strongest macro-F1 at 40\% corruption and remains clearly better than CE; its MBE lies between the best isolated geometry variant and the bounded-score baseline in both blocks. The ablation therefore supports complementary component roles: compactness improves balanced error and ranking, reliability improves probability quality, and separation most directly sharpens balanced decisions.

\begin{table}[H]
\centering
\caption{Component ablation on CIFAR-100 under long-tail imbalance with $\rho=100$ and 20\% or 40\% random incorrect-label replacement. Values are percentages reported as mean $\pm$ standard deviation over three seeds. All geometry-enabled variants use the fixed dataset-level coefficient $\eta=0.3$. ``Bounded'' denotes the prior-adjusted density-power score, and ``Reliab.'' denotes detached reliability weighting in the compactness term. Best values within each corruption block are shown in bold.}
\label{tab:ablation}
\setlength{\tabcolsep}{3.1pt}
\scriptsize
\begin{tabular}{lccccccc}
\toprule
Variant & Bounded & Reliab. & Compact. & Separ. & MBE $\downarrow$ & Macro-F1 $\uparrow$ & Macro-AUPRC $\uparrow$ \\
\midrule
\multicolumn{8}{c}{\textit{(a) 20\% random incorrect-label replacement}} \\
\midrule
CE reference & -- & -- & -- & -- & 76.37 $\pm$ 0.71 & 20.45 $\pm$ 0.61 & 19.19 $\pm$ 0.40 \\
Bounded score only & \yes & -- & -- & -- & 76.08 $\pm$ 0.21 & 21.58 $\pm$ 0.26 & 21.15 $\pm$ 0.51 \\
Score + separation & \yes & -- & -- & \yes & $\best{74.61 \pm 0.89}$ & $\best{24.11 \pm 0.52}$ & 21.67 $\pm$ 0.33 \\
Score + uniform compactness & \yes & -- & \yes & -- & 74.72 $\pm$ 0.27 & 22.79 $\pm$ 0.62 & 21.97 $\pm$ 0.61 \\
Score + reliability compactness & \yes & \yes & \yes & -- & 74.85 $\pm$ 0.90 & 22.80 $\pm$ 0.74 & $\best{22.22 \pm 0.46}$ \\
Full \BARGE{} & \yes & \yes & \yes & \yes & 75.49 $\pm$ 0.96 & 23.11 $\pm$ 2.10 & 22.10 $\pm$ 0.63 \\
\midrule
\multicolumn{8}{c}{\textit{(b) 40\% random incorrect-label replacement}} \\
\midrule
CE reference & -- & -- & -- & -- & 84.27 $\pm$ 0.81 & 12.49 $\pm$ 1.64 & 11.76 $\pm$ 0.57 \\
Bounded score only & \yes & -- & -- & -- & 82.40 $\pm$ 0.17 & 14.32 $\pm$ 1.46 & 14.20 $\pm$ 1.04 \\
Score + separation & \yes & -- & -- & \yes & $\best{81.96 \pm 0.93}$ & 14.52 $\pm$ 1.57 & 14.49 $\pm$ 0.40 \\
Score + uniform compactness & \yes & -- & \yes & -- & 81.98 $\pm$ 0.51 & 13.43 $\pm$ 0.95 & $\best{14.90 \pm 0.39}$ \\
Score + reliability compactness & \yes & \yes & \yes & -- & 81.99 $\pm$ 1.43 & 13.97 $\pm$ 1.10 & 13.92 $\pm$ 1.19 \\
Full \BARGE{} & \yes & \yes & \yes & \yes & 82.26 $\pm$ 0.88 & $\best{15.60 \pm 0.96}$ & 14.18 $\pm$ 1.59 \\
\bottomrule
\end{tabular}
\end{table}

\subsection{Hyperparameter sensitivity and computational cost}

\Cref{tab:eta} summarizes the dataset-level validation used to select the geometric coefficient.

\begin{table}[H]
\centering
\caption{Dataset-level selection of the geometric coefficient $\eta$. Each candidate entry reports its average within-setting validation rank across clean and noisy tuning conditions and seeds $\{1001,1002,1003\}$; lower is better. The final column reports the value selected and fixed for all final evaluations on that dataset.}
\label{tab:eta}
\setlength{\tabcolsep}{7pt}
\small
\begin{tabular}{lcccccc}
\toprule
& \multicolumn{5}{c}{Candidate coefficient $\eta$ (average validation rank)} & \\
\cmidrule(lr){2-6}
Dataset & $\eta=0.03$ & $\eta=0.10$ & $\eta=0.30$ & $\eta=0.50$ & $\eta=1.00$ & Selected $\eta$ \\
\midrule
CIFAR-10     & 4.02 & 3.21 & 2.65 & 2.71 & $\best{2.42}$ & 1.00 \\
CIFAR-100    & 3.17 & 3.00 & $\best{2.65}$ & 2.69 & 3.50 & 0.30 \\
Tiny ImageNet & 3.06 & $\best{2.06}$ & 3.17 & 3.28 & 3.44 & 0.10 \\
\bottomrule
\end{tabular}
\end{table}

The response in \cref{tab:eta} is not sharply peaked. CIFAR-100 differs by 0.04 rank units between \(\eta=0.30\) and 0.50, while CIFAR-10 favors stronger geometry and Tiny ImageNet favors a smaller coefficient. On Tiny ImageNet, the noisy conditions alone also select \(\eta=0.10\), with average rank 1.83. These differences support dataset-level validation of \(\eta\), while the broad response indicates limited sensitivity within the tested grid. The exponent \(\beta_C\) remains fixed throughout. Measured training cost is reported in \cref{tab:cost}.

\begin{table}[H]
\centering
\caption{Measured training cost on NVIDIA A100 GPUs. Time is the mean wall-clock duration of a 200-epoch run, and peak memory is the allocated GPU memory. The CIFAR-10/100 measurements average 24 clean setting--seed runs per method, while the Tiny ImageNet measurements average 12. All methods use the same inference procedure and introduce no loss-specific inference overhead.}
\label{tab:cost}
\setlength{\tabcolsep}{8pt}
\small
\begin{tabular}{lcccc}
\toprule
& \multicolumn{2}{c}{CIFAR-10/100} & \multicolumn{2}{c}{Tiny ImageNet} \\
\cmidrule(lr){2-3}\cmidrule(lr){4-5}
Method & Time (min) & Peak memory (GiB) & Time (min) & Peak memory (GiB) \\
\midrule
CE      & 10.46 & 2.70 & 48.29 & 10.55 \\
LA      & 11.05 & 2.70 & 46.34 & 10.55 \\
GCA     & 10.33 & 2.70 & 48.03 & 10.55 \\
\BARGE{} & 11.20 & 2.70 & 47.90 & 10.55 \\
\bottomrule
\end{tabular}
\end{table}

On CIFAR, \BARGE{} adds 7.1\% training time relative to CE and no measurable peak-memory increase at the reported precision. On Tiny ImageNet it is slightly faster than CE in the measured runs and uses the same peak memory. Runtime variation at this scale is influenced by data loading and shared-server load, so the relevant conclusion is parity within a small margin, not a speed advantage. The empirical cost agrees with the complexity analysis: the class Gram matrix is modest for \(C\leq200\), and the method introduces no second network or sample-pair matrix.

\section{Discussion}
\label{sec:discussion}

\paragraph{What the clean results establish.}
The clean experiments show that \BARGE{} does not purchase its bounded response by abandoning the balanced target. LA remains the strongest clean baseline, as expected from a direct prior correction. \BARGE{} is second overall and achieves the lowest error in four settings, including CIFAR-10 long-tail imbalance and step imbalance on both CIFAR-100 and Tiny ImageNet at \(\rho=1000\). The method is therefore viable as a balanced classifier even before noise is introduced. The settings where LA or CB wins also matter: additional geometry is not automatically beneficial when prior correction already matches the dominant source of error.

\paragraph{Why the advantage is clearer under label noise.}
Under corruption, the proper score targets the observed-label posterior and the observed prior, not their clean counterparts. The advantage cannot therefore be attributed to the clean balanced-consistency result. It is better explained by how the objective limits the action of a questionable label. The classification score cannot grow without bound, its risk perturbation is bounded at a fixed predictor, and its logit gradient vanishes under confident contradiction. The feature term uses the same adjusted target probability to determine relative geometric influence within the observed class. Subject to the within-class support condition in \cref{sec:theory}, a contradictory label is less able to dominate either the classifier update or the direction of its feature. This mechanism is consistent with the widening advantage at 20\% corruption, the macro-F1 and macro-AUPRC gains on all three datasets, and the tail-recall gains on CIFAR. Tiny ImageNet shows that it need not recover the rarest classes. The controlled study in \cref{sec:ablation} supplies one direct empirical link: class-equal angular compactness improves MBE and macro-AUPRC in all six matched comparisons with the bounded-score-only loss. Reliability weighting primarily improves probability quality, while separation is the strongest isolated contribution to MBE.

\paragraph{Balanced error and calibration tell different stories.}
The best ECE does not always accompany the best balanced error or proper-score metrics, as illustrated by the CIFAR-100 results. ECE summarizes agreement between confidence and correctness inside bins; it does not reward recovery of rare classes and can be small for uniformly cautious predictions. NLL and Brier score preserve more information about the predictive distribution, while macro-F1 and tail recall expose minority recognition. The joint metric view is therefore essential. \BARGE{} is strongest where these views agree, particularly at 20\% corruption, in the CIFAR-100 NLL/Brier comparisons, and across all six Tiny ImageNet probability columns. Its poor Tiny ImageNet tail recall prevents those probability gains from being mistaken for universal minority recovery.

\paragraph{What bounded influence does not imply.}
The method is not claimed to solve arbitrary noisy-label learning. When the model is uncertain, an incorrect label still contributes. When all examples in an observed class are unreliable, within-class normalization cannot make their total class contribution vanish. Corruption also changes \(\widehat{\bm\pi}\); even the uniform replacement process used here makes the observed class distribution less imbalanced than the clean construction. The theory establishes the clean balanced target and separate influence bounds under supplied-label contamination. It does not make the corrupted population minimizer equal to the clean one. The experiments show that these influence properties are useful for random incorrect-label replacement; they do not constitute a universal noise-tolerance theorem.

\section{Limitations}
\label{sec:limitations}

The evaluation uses three image benchmarks, one backbone, synthetic imbalance profiles, and one controlled label-corruption process. Tiny ImageNet extends the class count and resolution, but higher-resolution data, naturally occurring annotation errors, biomedical cohorts, and multiple modern architectures are needed to establish external validity. The separation term scales quadratically in \(C\). It is inexpensive at 200 classes, but sampled or blockwise approximations may be needed for vocabularies with tens of thousands of classes. Finally, the empirical prior is formed from observed labels. Estimating a clean prior without losing the single-stage character of the method remains an open direction for structured noise.

\section{Conclusions}
\label{sec:conclusion}

Class rebalancing should strengthen trustworthy minority evidence without granting unchecked influence to unreliable labels. \BARGE{} realizes this principle through a single-stage objective that couples balanced prior correction, bounded proper prediction, and reliability-guided angular geometry. Under clean supervision and the true class prior, its adjusted score identifies the training posterior while the raw logits recover the balanced Bayes ordering. Under label corruption, bounded risk perturbation and a redescending gradient limit the predictive influence of confidently contradicted labels. The same adjusted probability controls label-conditioned feature contraction, while class-equal normalization and one-sided separation organize the representation without altering inference. The empirical results demonstrate the value of this coupling. Across clean imbalanced settings, \BARGE{} remains competitive with strong specialized baselines, ranks second overall, and obtains the best average rank on Tiny ImageNet. Under 20\% and 40\% incorrect-label replacement, it achieves the lowest mean balanced error in all six dataset--corruption settings, reducing the six-setting average from 72.32\% for the strongest competing method to 70.00\%. It also leads macro-F1 and macro-AUPRC across these settings. Controlled ablations show that class-equal angular compactness contributes beyond the bounded classification score, while reliability weighting and classifier separation provide complementary control over discrimination and probability quality. The remaining tail-recall limitation on Tiny ImageNet identifies an important boundary: within-class normalization cannot recover a rare class when all available labels for that class are unreliable.

\BARGE{} is therefore most suitable when class frequencies are highly skewed, annotation quality is uncertain, and neither a reliable noise-transition model nor a clean-label correction procedure is available. It requires one network, introduces no loss-specific inference overhead, and adds only one validation-selected coefficient. More broadly, the results establish bounded predictive influence and reliability-guided representation learning as complementary design principles for classification under simultaneous class imbalance and label uncertainty.

\subsubsection*{Broader Impact Statement}

Methods for imbalanced classification can improve recognition of underrepresented classes in domains where missed rare events are consequential. The same setting requires care: synthetic imbalance and label corruption do not capture all deployment shifts, and aggregate benchmark improvements do not establish safety for any demographic or clinical subgroup. Applications with consequential decisions should therefore pair class-balanced evaluation with domain-specific validation, subgroup analysis, calibration checks, and human oversight. \BARGE{} changes the training objective but does not replace these safeguards.

\bibliographystyle{tmlr}
\bibliography{refs}

\end{document}